\documentclass[letterpaper]{article} % DO NOT CHANGE THIS
\usepackage{aaai2026}  % DO NOT CHANGE THIS
\usepackage{amssymb,amsmath,amsthm}
\usepackage{mathtools,amsthm}

\usepackage{times}  % DO NOT CHANGE THIS
\usepackage{helvet}  % DO NOT CHANGE THIS
\usepackage{courier}  % DO NOT CHANGE THIS
\usepackage[hyphens]{url}  % DO NOT CHANGE THIS
\usepackage{graphicx} % DO NOT CHANGE THIS
\usepackage{natbib}  % DO NOT CHANGE THIS AND DO NOT ADD ANY OPTIONS TO IT
\usepackage{caption} % DO NOT CHANGE THIS AND DO NOT ADD ANY OPTIONS TO IT
\usepackage{newfloat}
\usepackage{listings}
\DeclareCaptionStyle{ruled}{labelfont=normalfont,labelsep=colon,strut=off} % DO NOT CHANGE THIS
\usepackage[ruled,vlined, linesnumbered]{algorithm2e}
\usepackage{multirow}
\usepackage{tabularx}
\usepackage{graphicx}
\usepackage{multicol}
\usepackage{longtable}
\usepackage{appendix}
\usepackage{tcolorbox}
\usepackage{xcolor}
\usepackage{amsmath}
\usepackage{listings}
\usepackage{booktabs}
\usepackage{latexsym}

\newtheorem{theorem}{Theorem}
\newtheorem{lemma}{Lemma}
\newtheorem{definition}{Definition}
\SetKwInOut{Parameter}{Parameters} % Defines a new keyword named Parameter

\newcommand{\softmax}{\operatorname{softmax}}

\definecolor{darkgreen}{rgb}{0.0, 0.5, 0.0}
\definecolor{lightblue}{RGB}{173,216,230}
\definecolor{lightred}{RGB}{255,182,193}
\definecolor{lightgreen}{RGB}{173,255,47}
\definecolor{lightyellow}{RGB}{255,255,204}
\definecolor{violet}{RGB}{90, 19, 242}

\newcommand{\logits}{\text{LLM}}
\newcommand{\clip}{\text{clip}}
\newcommand{\norm}[1]{\left\lVert#1\right\rVert}

\newcommand{\tz}{\tilde{z}}

\newcommand{\tbz}{\mathbf{\tz}}

\newcommand{\eos}{\texttt{<eos>}}

\newcommand{\subsetsize}{s}

\definecolor{codegreen}{rgb}{0,0.6,0}
\definecolor{codegray}{rgb}{0.5,0.5,0.5}
\definecolor{codepurple}{rgb}{0.58,0,0.82}
\definecolor{backcolour}{rgb}{0.95,0.95,0.92}
\definecolor{myhighlight}{RGB}{255, 230, 150}

\lstdefinestyle{mystyle}{
  commentstyle=\color{blue},
  keywordstyle=\color{black},
  basicstyle=\ttfamily\fontsize{8}{9.5}\selectfont,  %
  breakatwhitespace=true,         
  breaklines=true,                 
  captionpos=b,                    
  keepspaces=true,                 
  numbersep=3pt,                  
  showspaces=false,                
  showstringspaces=false,
  showtabs=false,                  
  tabsize=2,
  breakindent=0pt,
  numbersep=3pt,
  numbers=left,
  numberstyle=\tiny\ttfamily\color{codegray},
  aboveskip=-5pt,
  belowskip=-5pt,
   xleftmargin=0pt,
  xrightmargin=0pt,
  }
\title{Privacy Preserving In-Context-Learning Framework for Large Language Models}

\author{
    Bishnu Bhusal\textsuperscript{\rm 1,2}\thanks{Work partly done during Bishnu Bhusal's internship at SRI.}, 
    Manoj Acharya\textsuperscript{\rm 2}, 
    Ramneet Kaur\textsuperscript{\rm 2}, 
    Colin Samplawski\textsuperscript{\rm 2},\\
    Anirban Roy\textsuperscript{\rm 2}, 
    Adam D. Cobb\textsuperscript{\rm 2}, 
    Rohit Chadha\textsuperscript{\rm 1}, 
    Susmit Jha\textsuperscript{\rm 2}
}
\affiliations {
    \textsuperscript{\rm 1}University of Missouri, 411 S 6th St, Columbia, MO 65211, USA\\
     \textsuperscript{\rm 2}SRI International, 333 Ravenswood Avenue, Menlo Park, CA 94025, USA\\
     \{bhusalb, chadhar\}@missouri.edu \\
    \{manoj.acharya, ramneet.kaur, colin.samplawski, anirban.roy, adam.cobb, susmit.jha\}@sri.com 
    
}

\usepackage{bibentry}
\begin{document}

\maketitle
\begin{abstract}
  Large language models (LLMs) have significantly transformed natural language understanding and generation, but they raise privacy concerns due to potential exposure of sensitive information. Studies have highlighted the risk of information leakage, where adversaries can extract sensitive information embedded in the prompts.  In this work, we introduce a novel private prediction framework for generating high-quality synthetic text with strong privacy guarantees. Our approach leverages the Differential Privacy (DP) framework to ensure worst-case theoretical bounds on information leakage without requiring any fine-tuning of the underlying models. The proposed method performs inference on private records and aggregates the resulting per-token output distributions. This enables the generation of longer and coherent synthetic text while maintaining privacy guarantees. Additionally, we propose a simple blending operation that combines private and public inference  to further enhance utility. Empirical evaluations demonstrate that our approach outperforms previous state-of-the-art methods on in-context-learning (ICL) tasks, making it a promising direction for privacy-preserving text generation while maintaining high utility. Our code is available at \url{https://github.com/bhusalb/privacy-preserving-icl}.
\end{abstract}

\maketitle
\section{Introduction}
\label{introduction}

Large Language Models (LLMs) have enjoyed widespread success in many applications. Although they primarily obtain foundational knowledge through pre-training, most users tailor trained LLMs through prompt engineering. Compared to the resource-intensive optimization of model parameters during training, prompt engineering is typically performed via API calls, where prompts are progressively refined to achieve optimal downstream performance.

However, this workflow has data privacy risks as sensitive records can be exposed in prompts and responses.  For example, when LLMs are deployed with user data, such as clinical reports, incorporated into prompts, there is a risk that sensitive information can be inadvertently disclosed to non-relevant users \cite{wang2023decodingtrust}. An adversary can also extract sensitive user data in prompts via  “jailbreaks”, where even entire prompts or segments of prompts can be extracted verbatim from some attacks. A mitigation approach could be to scrub Personally Identifiable Information (PII) from prompts, but even with that, there have been cases, such as in linkage attacks, where a combination of secondary information is enough to link the record back to an individual \cite{powar2023sok}, leading to potential privacy violations.

Differential Privacy (DP) \cite{dwork2014algorithmic} has been previously used for protecting individual data as it ensures that sensitive information stays confidential while allowing meaningful insights to be drawn from the aggregated dataset. The U.S. Census Bureau's LEHD OnTheMap tool~\cite{census}, Google's RAPPOR system as part of Google Chrome~\cite{rappor},
% Apple's DP implementation~\cite{apple-white,apple-patent}, Microsoft's Telemetry collection~\cite{microsoft}
 and Facebook and Social Science One's release of election dataset~\cite{DVN/TDOAPG_2020,Evans_King_2023} are some noteworthy examples of industry-level deployment of this technology. To apply DP in LLM workflows, perhaps a simple and elegant way would be to transform the original prompt corpus into a semantically equivalent synthetic dataset. This synthetic version of the private dataset preserves the same overall patterns as the real data but contains no actual user records, making it safe to use for model training, inference, or external sharing without risking privacy breaches.

Current approaches to creating differentially private text with large language models fall into two main groups: private fine-tuning and private prediction \cite{dwork2018privacy}. Fine-tuning methods update model weights on the private data using a DP-SGD style algorithm \cite{abadi2016deep}.  Once fine-tuned, the model generates synthetic text directly. This method often produces high-quality outputs but requires extensive compute for training as well as full access to the model parameters. On the other hand, private prediction based methods only rely on test-time inference instead of fine-tuning the model \cite{van2020trade,majmudar2022differentially}. In these approaches, noise is added to the model’s output distribution so that each generated token satisfies differential privacy. As they avoid any form of model training, synthetic examples can be produced on demand. However, because a separate privacy cost is incurred for each generated token, the overall privacy budget accumulates rapidly which limits this approach to generating a smaller amount of synthetic text \cite{tang2024privacypreserving}.

In this paper, we introduce a private prediction technique that generates large-scale synthetic text %at a larger scale than that of the current methods 
while preserving strong DP guarantees. Similar to previous approaches for producing synthetic data privately, our method also runs inference over multiple disjoint subsets of the private data and then aggregates the per-token output distributions under a DP mechanism to produce synthetic examples. However, we introduce key improvements that measurably improves inference efficiency and data utility, while providing same privacy guarantees. Our DP mechanism is simpler and easy to implement.
%simplifies the current mechanism aimed at improving inference efficiency and data utility. 
Our main contributions are as follows:

\begin{itemize}
    \item \textbf{Novel Aggregation Approach:} Our method introduces a simple yet effective aggregation strategy that separates from prior threshold-based or heuristic-heavy methods. Our approach combines logit distributions obtained from disjoint private subsets and public prompts using differentially private clipping and averaging, ensuring privacy guarantees via composition. This simple yet principled design avoids the need for delicate calibration, reduces computational overhead and offers clear theoretical analysis and practical deployment.
    \item \textbf{Improved Efficiency:}  Existing techniques often rely on randomly sampling new subsets of demonstrations for each generation step, which necessitates re-initializing the KV-cache. This repeated recomputation of the prefix is computationally intensive and impractical for real-world usage. In contrast, our approach uses a fixed, disjoint subset of input data to generate synthetic examples. By leveraging composition and reusing cached prefix encoding, we incur only a linear computational cost rather than quadratic with respect to the number of synthetic generated tokens thus enabling efficient decoding.
    % \item \textbf{Enhanced Synthetic Text Quality:} By leveraging public LLM outputs, the method yields higher-quality synthetic text generation under differential privacy constraints. \manoj{better ft accuracy}
    % \item \textbf{Robust against attacks:} Traditional approaches, that add Laplace or Gaussian noise during generation are often vulnerable to floating-point exploits. Our system relies solely on randomness in the token-sampling process, making it inherently resistant to such attacks.
    \item \textbf{Privacy-Preserving ICL via Synthetic Demonstrations:} Our approach first generates synthetic examples from the private dataset using a differentially private (DP) algorithm. These synthetic generations are then used as few-shot demonstrations during LLM inference within the in-context learning (ICL) framework. This two-stage design enables the use of private data for ICL without compromising privacy, and supports high-utility predictions while maintaining formal DP guarantees. Empirically, our method delivers improvements in ICL accuracy across five diverse benchmark tasks, surpassing existing baselines, while offering computational efficiency and formal privacy guarantees.
    
\end{itemize}

\begin{figure*}[!ht]
    \centering
    \includegraphics[width=1\linewidth]{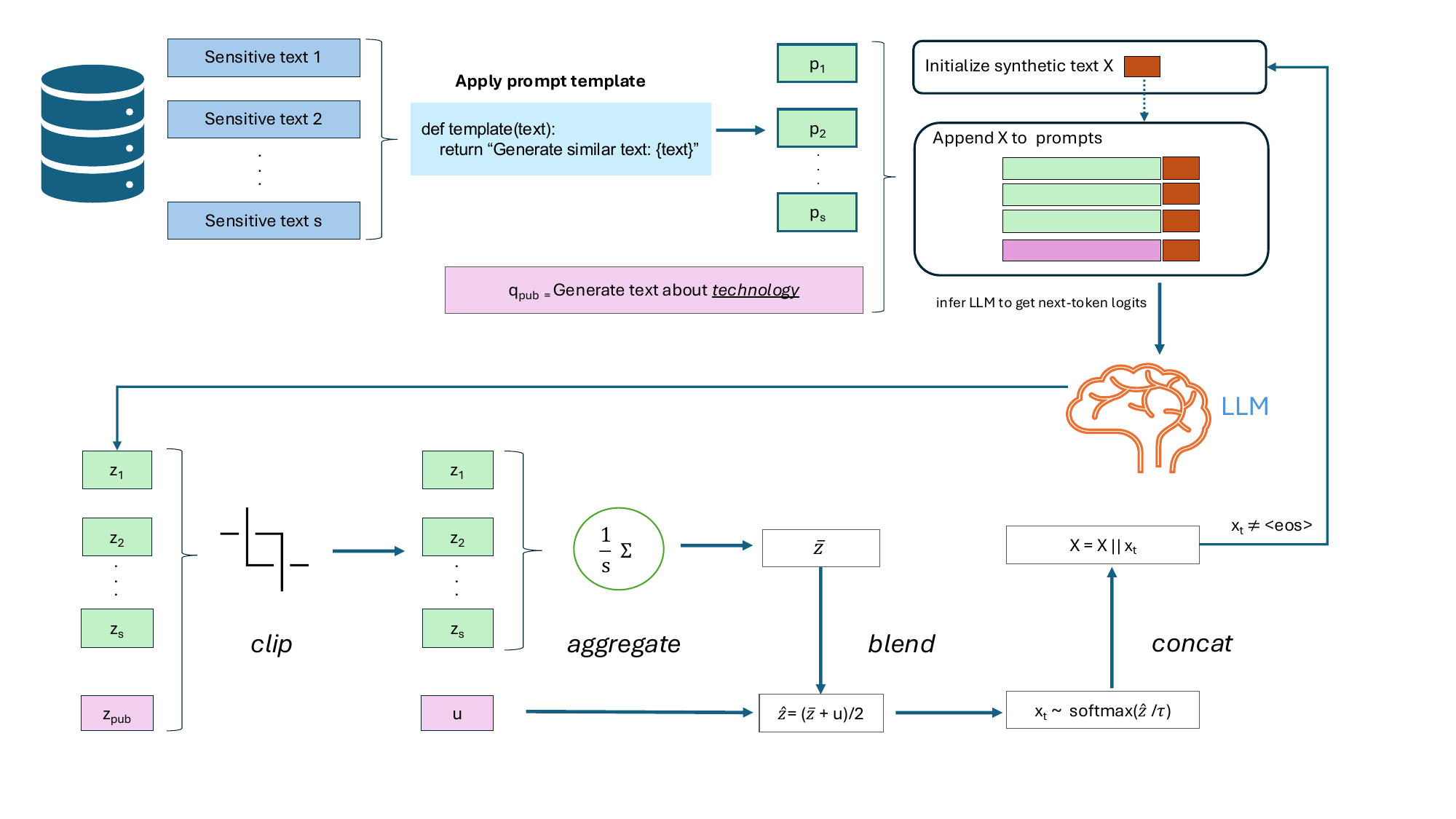}
    \caption{Overview of the proposed privacy-preserving synthetic text generation framework. A set of demonstrations is first sampled from the private dataset to construct prompts for next-token generation. These prompts are passed to an LLM to produce token-wise logits $(z_1, z_2, ... z_s)$, while a parallel public prompt yields a public logit vector $z_\text{pub}$. All logits are clipped to bound sensitivity. Then, only private logits are aggregated to compute $\bar{z} = \text{clip\_aggregate}(z_1, z_2, ... z_s)$. This aggregated private logit is then blended with clipped public logit $u$ and a token $x_t$ is sampled from the resulting temperature-scaled softmax distribution. The sampled token is appended to the synthetic sequence X, and this process is repeated until the end-of-sequence $\eos$ token is emitted.}
    \label{fig:approach}
\end{figure*}

\section{Related Work}

We focus on the  privacy-preserving in-context learning (ICL) framework, where large language models (LLMs) can perform downstream tasks effectively using only a few demonstrations, without requiring fine-tuning, as demonstrated in~\citet{BrownMRSK20}. Among the notable efforts in privacy-preserving ICL, \citet{panda2023differentially} introduce a differentially private (DP) inference mechanism by constructing a consensus over ensembles of queries with disjoint demonstrations. Although their method satisfies DP guarantees, it incurs a privacy cost for each query, thereby limiting the number of queries that can be answered under a given privacy budget. \citet{duan2023flocks} also explore private ensembling but rely on the availability of unlabeled public data, which is labeled using a teacher ensemble via ICL. This reliance on public data contrasts with our approach, which operates solely on private data, making it more suitable for sensitive domains such as healthcare or industrial applications where public datasets with similar distributions may not be accessible. Furthermore, both works primarily target text classification, while our method extends to a broader range of tasks, as demonstrated in our experiments.

Our contribution aligns with the general domain of synthetic text generation under privacy constraints, diverging from approaches that rely heavily on private fine-tuning~\citep{yu2022differentially, LiTLH21}. Inspired by the capabilities of LLMs~\citep{BrownMRSK20}, our method leverages their generation abilities in a DP-compliant manner by privately aggregating generation probabilities over disjoint subsets of private demonstrations. This technique draws conceptual similarity to the PATE framework~\citep{PapernotAEGT17, PapernotSMRTE18}, which generates private models by training on public data labeled by an ensemble of teacher models trained on disjoint private data. Extensions of PATE to text generation include SeqPATE~\citep{tian2022seqpate} and Submix~\citep{ginart2022submix}, which introduce domain-specific adaptations. 

Another direction is DP decoding, as proposed by~\citet{majmudar2022differentially}, which combines LLM predictions with uniform distributions. However, these methods generally require private training on the sensitive data, unlike our lightweight approach. Compared to private fine-tuning techniques for synthetic data generation~\citep{yue2023synthetic, mattern-etal-2022-differentially, mireshghallah2023, carranza2023privacypreserving}, our method avoids computationally intensive fine-tuning and is well-suited for scenarios where few-shot ICL suffices. While private fine-tuning may be preferable for generating large volumes of synthetic data, our approach balances efficiency and effectiveness in low-data settings.

Beyond ICL, privacy-preserving efforts in natural language generation encompass several techniques. Word-level noise injection and metric local differential privacy (LDP) have been used for sanitized text generation~\citep{feyisetan2020privacy, xu-etal-2020-differentially, carvalho2021tem, du2023sanitizing}. Other methods, such as those by~\citet{mattern-etal-2022-limits} and~\citet{utpala2023locally}, apply LDP directly to full documents through fine-tuning or zero-shot prompting followed by sanitization. Private fine-tuning remains prominent in synthetic data generation: \citet{yue2022synthetic} utilize DP-SGD~\citep{abadi2016deep} to fine-tune LLMs, while~\citet{kurakin2023harnessing} demonstrate improvements via parameter-efficient fine-tuning like LoRA~\citep{hu2021lora}. Two-stage fine-tuning approaches have also been proposed~\citep{wu2024prompt}, and similar ideas have been extended to structured data~\citep{tran2024differentially}. Another line of work focuses on private prediction~\citep{dwork2018privacy}, where privacy is guaranteed only for outputs, often via subsample-and-aggregate techniques~\citep{nissim2007smooth}, as used in PATE~\citep{papernot2018scalable}. Applied to synthetic text, these ideas involve per-token privacy accounting~\citep{tang2024privacypreserving, hong2023dp}, but suffer from limited utility due to the high privacy cost of each token. Other adaptations of private prediction to LLMs~\citep{majmudar2022differentially, ginart2022submix, flemings2024differentially} have not focused on synthetic generation. Lastly, private filtering methods operate on entire LLM responses and rely on matching public data via embedding similarity or keyword selection~\citep{yu2024privacy, xie2024differentially, wu2023privacy}, but lack adaptability to new data distributions.

Compared to previous approaches, by forgoing the Sparse Vector Technique (SVT) and its associated threshold selection procedures, our approach reduces both algorithmic complexity and runtime overhead.  This also further reduces the need to tune additional hyper-parameters  and user essentially need to set a single privacy budget, making deployment and tuning more straightforward.  Finally,  we obtain tighter worst‐case privacy guarantees which results in smaller noise scale which yields a more robust privacy mechanism with only minimal impact on downstream model utility.

\section{Background}

\subsection{In-Context Learning}
In-context learning (ICL) leverages a pre-trained model to utilize its existing knowledge by conditioning it on a sequence of demonstration examples without any further gradient updates to the model weights \cite{brown2020language,liu2021pretrainpromptpredictsystematic}. During inference, the model receives several such input–label demonstration pairs which follow a consistent format, followed by a novel test input subjected to the same pattern. The model has to then autoregressively predict the correct label for the final prompt in a few-shot manner, effectively learning the task from the provided context rather than solely relying on the information stored in model parameters \cite{dong2024surveyincontextlearning}. 
Recently, ICL has proven to be versatile across diverse NLP tasks ranging from text classification and question answering, especially as model scale increases, highlighting the emergent capabilities of LLMs \cite{wei2022emergent}.

\subsection{Differential Privacy}
Let $\mathcal{D}$ denote the set of all prompt datasets. A mechanism is a randomized algorithm that operates on data sets from $\mathcal{D}$. Two datasets $D, D' \in \mathcal{D}$ are neighboring if they differ by a single prompt (i.e., one is obtained from the other by adding or removing exactly one prompt). This follows the standard \emph{add/remove} definition of neighboring datasets in differential privacy.

 \begin{definition}[Differential Privacy (DP) \citep{DworkKMMN06}]
  A randomized algorithm $\mathcal{A}$ is  ($\epsilon$,$\delta$)-differentially private if for any two neighboring datasets $D, D' \in \mathcal{D}$ and for any set $\mathcal{S}$ of possible outputs: 
$
\textstyle{\Pr[\mathcal{A}(D) \in \mathcal{S}] \leq e^{\epsilon}\,\Pr[\mathcal{A}(D') \in \mathcal{S}] +\delta}
$.

Here, $\epsilon > 0$ controls the privacy loss, where a smaller value implies stronger privacy, and $\delta \geq 0$ represents the probability of failure, allowing for a small chance that the guarantee does not hold.
\end{definition}

\section{Problem Definition}
We address the problem of privacy-preserving in-context-learning (ICL) for large language models. Consider a private dataset $\mathcal{D}_{\text{priv}} = \{d_1, \cdots, d_n\}$ where each data point $d_i$ consists of a text-label pair, i.e., $d_i = (t_i, y_i)$. Our task is to protect the privacy of these data points from an adversary, whose goal is to either directly access or infer private information about them. To ensure this, the output of the learning process must satisfy differential privacy (DP) with respect to $\mathcal{D}_{\text{priv}}$. Specifically, for any two neighboring datasets differing in only a single entry $d_i$ the output distribution must be statistically indistinguishable.

We formally define a single instance of in-context-learning (ICL) as the following. Given a pre-trained language model that produces token-level output logits  $\logits(x_n \mid x_1, \cdots, x_{n-1})$, where each $x_i$ is a token in a vocabulary $\mathcal{V}$: $x_i \in \mathcal{V}$. The model is provided with a query input $q$ and a set of demonstration examples $D_{\text{dem}}$. The model then produces a predicted label $y$ for the query $q$ according to the function: 
\begin{align}
\label{eq:defofproblem}
y := f_{{\logits}}(D_{\text{dem}}, q).
\end{align}

\begin{algorithm}[tb]
\DontPrintSemicolon
\caption{Private Synthetic Examples Generation}
\label{alg:main}
\Parameter{ $\epsilon > 0$, $\delta \in [0,1]$, $\logits$,  private prompt set $P$ of expected size $\subsetsize$, public prompt $q_{\text{pub}}$, clipping threshold $c > 0$, temperature $\tau$, per‐iteration sensitivity bound $\Delta$, and max number of tokens to generate $T$}
\KwIn{Subset of sensitive prompts $P \in \mathcal{D}_{\text{priv}}$; each prompt contains a sensitive example.}
\KwOut{A Synthetic Example $X$}
$X \leftarrow \emptyset$\;
$\Delta \leftarrow \frac{c}{2\subsetsize}$ \;
$\tau \leftarrow \frac{2\Delta\sqrt{2T \ln(1/\delta)}}{\varepsilon}$\;
\For{$t = 1, \ldots, T$}{
    $Z \leftarrow \{\logits(p \| X) \mid p \in P\}$\;
    $\bar{z} \leftarrow \frac{1}{\subsetsize}\sum_{z \in Z} \clip_c(z)$\;
    $u \leftarrow \clip_c(\logits(q_{\text{pub}} \| X))$\;
    $\hat{z} \leftarrow (\bar{z} + u)/2$\;
    $x_t \sim \softmax(\hat{z}/\tau)$\;
    \If{$x_t = \eos$}{
        \textbf{break}\;
    }
    $X \leftarrow X \| x_t$\;
}
\Return{$X$}\;
\end{algorithm}

\section{Approach}

We present a private prediction protocol for next-token prediction in the Figure \ref{fig:approach}. Our approach follows a two-step framework for in-context-learning (ICL) with differential privacy:
\begin{enumerate}
    \item Generate synthetic examples from the private dataset $\mathcal{D}_{\text{priv}}$ using a DP algorithm.
    \item Use these synthetic generations as ICL demonstrations during LLM inference.
\end{enumerate}

This approach allows offline pre-processing and, due to the post-processing property of DP, incurs no additional privacy cost during inference.

Below, we begin by outlining standard LLM inference in a typical decoder based model, then introduce our differentially private prediction method for generating synthetic examples. Finally, we present the formal privacy guarantees provided by our algorithm.

\subsection{LLM Inference} 
Any decoder-only LLM such as GPT~\cite{gpt} and Llama~\citep{liu2022llama} takes an input prompt and generates a sequence of token indices in auto-regressive manner. The model maps each generated token $x_t$ to a logit vector $z\in\mathbb{R}^{V}$, where $V$ is the token vocabulary size. This process involves initializing the prompt sequence $X$ with the instruction phrase $p$, and repeating the following steps: (a) compute logits for the next token $x_t$ as 
$z_t=\logits(x_t)$, (b) sample the next token $x_t\sim\softmax(z_t/\tau)$ for the temperature hyper-parameter $\tau>0$, and (c) append $x_t$ to $X$. The process stops when $x_t$ is the end-of-sequence token $\texttt{<eos>}$, indicating the end of the response. Here, $\softmax(z_t/\tau)$ is the distribution that assigns probability proportional to $\exp(z_t/\tau)$ to the $t_{th}$ token, and $\tau > 0$ is a hyper-parameter that flattens or sharpens the distribution.

\subsection{Our Algorithm} 
\label{sec:our_algo}
We solve the proposed problem by generating synthetic examples $X$ while satisfying $(\epsilon, \delta)$-DP on the private dataset $\mathcal{D}_{\text{priv}}$  without fine-tuning the underlying LLM. A simple approach to generating synthetic versions of sensitive text is to use an LLM-based prompting pipeline. This involves defining a prompt function that generates synthetic samples given a label $y$, instructions for a task and a list of private data samples $p_i$ contained in the subset of sensitive prompts $P$. For instance, a prompt like \textit{``Generate similar text to: $<$sensitive text$>$''} might be used. However, such naive prompting can pose serious privacy risks, as the generated output may not only preserve the semantics of the input but also inadvertently reproduce sensitive fragments of the original text.

Algorithm \ref{alg:main} describes our method for privately generating a dataset of synthetic examples X from a dataset of sensitive prompts $\mathcal{D}_{\text{priv}}$. Our innovation lies in the fact that don't use a single prompt but instead a subset of prompts $P$ of size $\subsetsize$ and run LLM parallel inference  on each prompt. Each such inference generates a token $x_t$ with corresponding logits $\hat z$. An average of all  logit vectors  across the batch define the distribution from which the next token is selected. Before averaging, all logit vectors $z_i\in\mathbb{R}^{V}$ are clipped and re-centered using the function:

\begin{equation}
    \clip_c(z_i) = \max\{-c, z_i - \max_j {\{z_j\}} + c\} \label{eq:clip}
\end{equation}

\begin{table*}[t!]
    \centering
    \begin{tabular}{llllrrrrr}
    \toprule
\multicolumn{1}{c}{$\varepsilon$} & Method & Shots & Model & \multicolumn{1}{c}{AGNews} & \multicolumn{1}{c}{DBPedia} & \multicolumn{1}{c}{TREC} & \multicolumn{1}{c}{MIT-G} & \multicolumn{1}{c}{MIT-D} \\ \midrule
0 & Zero shot & 0 & - & $24.8_{0.0}$ & $12.0_{0.0}$ & $28.4_{0.0}$ & $29.6_{0.0}$ & $28.8_{0.0}$ \\ \midrule
\multirow{4}{*}{$\infty$} & Real data & 4 & - & $75.3_{3.0}$ & $73.6_{0.3}$ & $34.9_{5.0}$ & $56.0_{2.0}$ & $83.1_{5.3}$ \\
 & \citet{tang2024privacypreserving} & 4 & GPT-3 babbage & $69.3_{4.8}$ & $82.3_{3.7}$ & $50.6_{6.9}$ & $54.4_{7.0}$ & \multicolumn{1}{c}{-} \\
 & \citet{amin2024private} & 4 & Gemma 1.1 2B IT & $\mathbf{76.8_{4.8}}$ & $72.3_{2.5}$ & $38.8_{6.0}$ & $47.7_{2.5}$ & $81.7_{2.4}$ \\
 & \textbf{Ours} & 4 & Gemma 1.1 2B IT & $73.5_{6.0}$ & $\mathbf{81.8_{4.4}}$ & $\mathbf{62.0_{6.3}}$ & $\mathbf{58.2_{2.3}}$ & $\mathbf{87.1_{2.7}}$ \\ \midrule
\multirow{4}{*}{$1$} & \citet{tang2024privacypreserving} & 4 & GPT-3 babbage & $64.1_{3.9}$ & $81.2_{3.0}$ & $50.7_{4.1}$ & $46.3_{7.8}$ & $69.2_{7.9}$ \\
 & \citet{tang2024privacypreserving} & 4 & Gemma 1.1 2B IT & $74.9_{3.8}$ & $\mathbf{80.9_{3.6}}$ & $36.7_{2.2}$ & $34.1_{9.3}$ & $78.7_{1.9}$ \\
 & \citet{amin2024private} & 4 & Gemma 1.1 2B IT & $75.9_{3.5}$ & $75.1_{0.5}$ & $39.2_{3.7}$ & $47.1_{6.0}$ & $\mathbf{84.5_{1.0}}$ \\
 & \textbf{Ours} & 4 & Gemma 1.1 2B IT & $\mathbf{79.5_{2.6}}$ & $76.8_{2.9}$ & $\mathbf{63.0_{2.1}}$ & $\mathbf{47.2_{0.5}}$ & $79.9_{2.5}$ \\
    \bottomrule
    \end{tabular}
    \caption{\small In-context-learning accuracy comparison where we report mean and standard deviation over three random samplings (equally many from each label for classification; fully random for extraction) of synthetic/real data. \textbf{(*) Note}: For the results using GPT-3 babbage only the top-100 logprobs for contextual calibration (only top 5 are available now) are used. While not directly comparable to Gemma model which uses logprobs over the full vocabulary, we report their results for context similar to \cite{amin2024private}. Best results for $\epsilon=\infty, \text{and } 1$ on Gemma 1.1 2B IT are in bold.}
    \label{table:main}
\end{table*}

where $c>0$ is the clipping threshold and $\max{\{z_j\}}$ is the maximum value for each logit vector. This clipping operation bounds the  entries for each logit vector within the range $[-c,c]$. This aids our privacy analysis, particularly in formalizing per-token sensitivity by enabling us to bound the $\ell_{\infty}$ norm of potential transformations applied to the logits. Importantly, clipping does not affect the outcome of the softmax operation, as softmax is invariant to uniform shifts in its inputs.

Since the  averaged logit vector is generated from private subset hence each token selected from this vector adds to the privacy budget. To minimize the privacy leakage, we also use generate an axillary token distribution from the same LLM
without access to the sensitive data. It uses a public prompt function that generates text for a given text of category $y_i$. At each iteration, a public token $u$  is generated by combining its clipped logits with the aggregated private logits $\bar z$ using simple averaging. According to our privacy analysis, clipping ensures that the influence of any individual private prompt is bounded by $\frac{c}{s}$ in each coordinate and after merging with public logits this is further reduced to $\frac{c}{2s}$. This reduction in sensitivity allows us to inject less noise and operate under a smaller privacy budget compared to \citet{amin2024private}. From a utility perspective, since both $\bar z$ and $u$ are clipped in a rank‐preserving manner, their arithmetic mean preserves token preferences common to both sources. As a result, tokens strongly supported by both private and public contexts receive the highest scores, while those favored by only one are suppressed.

\begin{lemma}[Exponential Mechanism \cite{mcsherry2007mechanism}]
\label{thm:exp}
Let $\mathcal{R}$ be a set of possible outputs and let $q: \mathcal{D} \times \mathcal{R} \to \mathbb{R}$ be a utility function such that for any two adjacent databases $D$ and $D'$ (i.e., differing in one record), the sensitivity of $q$ satisfies:
\[
\Delta = \max_{r \in \mathcal{R}} |q(D, r) - q(D', r)|.
\]
The Exponential Mechanism $\mathcal{M}_E$ selects an output $r \in \mathcal{R}$ with probability proportional to:
\[
\Pr[\mathcal{M}_E(D) = r] \propto \exp\left( \frac{q(D, r)}{\tau} \right).
\] Where, $\tau = \frac{2\Delta}{\epsilon}$
Then $\mathcal{M}_E$ is $\epsilon$-differentially private.
\end{lemma}

At each iteration, we ensure differential privacy by selecting the new token using the exponential mechanism. Furthermore, using the composition property of DP  \cite{dwork2014algorithmic}, we guarantee that the entire sequence of upto $T$ generated tokens remains collectively is  also differentially private.

\subsection{Privacy Analysis}
\label{sec:privacy_analysis}

\begin{theorem}[Privacy of Algorithm \ref{alg:main}] \label{thm:main}
For all $\subsetsize > 0$, $\tau > 0$, $\epsilon>0$ and $\delta \in (0,1]$, Algorithm \ref{alg:main} satisfies $(\epsilon, \delta)$-differential privacy, where
\[\epsilon = \frac{c\sqrt{2T \ln(1/\delta)}}{\subsetsize \cdot \tau}\]
\end{theorem}
In this privacy bound, $c$ denotes the clipping threshold, $T$ is the number of composition steps, $s$ is the size of the sensitive subset, and $\tau$ is the temperature parameter used in softmax sampling.

% ---------------------------------------------------------

% \section{Proof of Theorem~\ref{thm:main}}

\subsubsection{Sensitivity analysis}
\label{sec:sensitivity}

We analyze the sensitivity of several functions used in Algorithm \ref{alg:main}.  Each such function is defined on a multiset $Z = \{z_1, z_2, \dots, z_n\} \subset \mathbb{R}^v$, where each $z_i \in \mathbb{R}^v$ is a logit vector. Consider a the function $\ell : (\mathbb{R}^v)^n \to \mathbb{R}^v$ corresponding to line 8 of Algorithm \ref{alg:main} defined as:
\[
\ell(Z) =  \frac{\bar{z} + u}{2}
\]
where $u \in \mathbb{R}^v$ is a public logit vector, and $\bar{z} = \frac1\subsetsize \sum_{z \in Z} \clip_c(z)$ is the mean of the clipped private logits. The operator  $\clip_c(\cdot)$ performs element-wise clipping with threshold $c >0$, and $s$ is the number of logit vectors in the private logit vectors.

\begin{lemma} \label{lem:sensitivity} The function $\ell$ has sensitivity $\Delta = \frac{c}{2\subsetsize}$.
\end{lemma}
\begin{proof} Let $Z, Z' \subseteq \mathbb{R}^v$ be neighbors i.e. they differ by a single record. Let $\tbz \in \mathbb{R}^v$ be the logit vector they do not have in common. We have
\begin{align*}
\norm{\ell(Z) - \ell(Z')}_\infty &= \frac{1}{2}\norm{\bar{z} + u - \bar{z'} + u}_\infty \\
&\leq \frac{1}{2\subsetsize}\norm{\sum_{z \in Z} \clip_c(z) - \sum_{z' \in Z'} \clip_c(z')}_\infty \\
&\leq \frac{1}{2\subsetsize} \norm{\clip_c(\tbz)}_\infty \\
&\leq \frac{c}{2\subsetsize} \qedhere\\ 
\end{align*}
\end{proof}

\begin{lemma}[Advanced Composition for Exponential Mechanism {~\cite{dwork2014algorithmic}}]
\label{thm:cmp}
Let $\mathcal{M}_1, \dots, \mathcal{M}_T$ be a sequence of randomized algorithms, where each $\mathcal{M}_t$ satisfies $\varepsilon'$-differential privacy (e.g., an Exponential Mechanism with sensitivity $\Delta$ and noise scale $\tau = \frac{2 \Delta}{\epsilon'}$).

Then the composed mechanism $\mathcal{M}(D) = (\mathcal{M}_1(D), \dots, \mathcal{M}_T(D))$ satisfies $(\varepsilon, \delta)$-differential privacy for any $\delta > 0$, where:
\[
\varepsilon = \sqrt{2T \ln(1/\delta)} \cdot \varepsilon' + T \varepsilon' (e^{\varepsilon'} - 1)
\]

In particular, for small $\varepsilon'$, this simplifies to:
\[
\varepsilon \approx \sqrt{2T \ln(1/\delta)} \cdot \varepsilon'
\]
\end{lemma}

\subsubsection{Privacy Accounting Over Full Generation}
\label{sec:together}

We now present the final composition of our privacy guarantees across the full execution of Algorithm~\ref{alg:main}.By Lemma~\ref{thm:exp}, each individual iteration of Algorithm~\ref{alg:main} satisfies $\frac{c}{\tau \cdot \subsetsize}$-differential privacy. Since $T$ is the maximum number of privately generated tokens for any input batch, there can be at most $T$ such distinct segments. Applying the advanced composition theorem (Lemma~\ref{thm:cmp}) over the at most $T$ sequential steps, we conclude that the complete execution of Algorithm~\ref{alg:main} satisfies $(\frac{c \sqrt{2T \ln(1/\delta)}}{\subsetsize \cdot \tau}, \delta)$-differential privacy. This composition analysis ensures that the entire token generation process, up to termination at the $\eos$ token or after $T$ private tokens, adheres to a formally bounded privacy guarantee under the Composition Theorem~\ref{thm:cmp}.

% ----------------------------------------------------------

\begin{figure}[!ht]
\centering
\begin{center}
\begin{tcolorbox}
\begin{lstlisting}[language=Python]
Classify the following examples:
#synthetic text 1
Input: The patient shows ...   
Answer: Diabetes
#...
#synthetic text n
Input: The patient has been ...
Answer: Hypertension
#evaluation text
Input: Patient experiences ...
Answer: 
\end{lstlisting}
\end{tcolorbox}
\end{center}
\caption{Example of our $k$-shot in-context-learning evaluation setup.}
\label{fig:icl-evaluation}
\end{figure}

\section{Experimental Setup}
\label{sec:experiments}
\paragraph{Datasets:} To measure the downstream utility of our privacy-preserving approach, we report the accuracy on test examples when prompted with the synthetic example generated using the proposed Algorithm \ref{alg:main}.

For evaluation, we follow the prior ICL work \cite{zhao2021calibrate} and use the following setup. For classification tasks, we use three datasets: the 4-way news classification dataset \emph{AGNews} \citep{zhang2015character},  the 6-way question classification dataset \emph{TREC} \citep{zhang2015character}, and the 14-way topic classification \emph{DBPedia} \citep{voorhees2000building}. We also evaluate on two information extraction tasks, namely \emph{MIT-G}, and \emph{MIT-D} \citep{liu2012conversational}. These are both slot-filling datasets with movie genre (MIT-G) and director name (MIT-D) as the slots to be filled. An illustrative example is shown in Figure \ref{fig:icl-evaluation}.

Since our method requires access to the full token probability distribution at each decoding step, we use the instruction-tuned (IT) variant of Gemma 1.1~\citep{gemmateam2024gemma}, a decoder-only LLM with two billion parameters, aligning with prior work by \citet{amin2024private}. However, this choice is made purely for experimental benchmarking. Our algorithm is model-agnostic and can be applied to any decoder-style LLM, provided access to full token-level logit distribution is available. We provide additional details for chosen hyper-parameters in the Appendix.

\section{Results}

Table \ref{table:main} presents our in‐context learning results on the five benchmark tasks. Our primary choice of  LLM is  Gemma-1.12B IT which provides access to model logits as opposed to closed models such as ``GPT3-baggage'' used by some previous works \cite{tang2024privacypreserving}. We evaluate and compare the utility of our model with varying level of privacy budget ($\varepsilon = 0, \infty, 1$). The first block with $\varepsilon = 0$ has the highest amount of privacy shows zero-shot performance (no ICL) is uniformly poor across tasks (e.g., 24.8\% on AGNews).

The second block with $\varepsilon = \infty$ examines when the least amount of noise is added in our DP framework. We observe that we perform better than the current SOTA \citet{amin2024private}'s baseline on Gemma-1.12B in four tasks, namely DBPedia, TREC, MIT-G and MIT-D, while lagging behind by only $3.3\%$ on AGNews. For most datasets: DBPedia, TREC, MIT-G and MIT-D, we also outperform~\citet{tang2024privacypreserving}'s baseline with GPT3 Babbage that uses only top-100 logprobs.

The third block with $\varepsilon = 1$ imposes stronger privacy constraints on synthetic data. We observe that we outperform both baselines by~\citet{tang2024privacypreserving}, and ~\citet{amin2024private} with Gemma-1.12B on three out of five tasks: AGNews, TREC, MIT-G. Similarly, for AGNews, TREC, MIT-G and MIT-D datasets, we outperform~\citet{tang2024privacypreserving}'s baseline with GPT3 Babbage that uses only top-100 logprobs.

\section{Effects of $k$-Shot and Privacy Budget Settings}

\begin{table}[h]
    \centering
    \small
    \begin{tabular}{llllll}
    \toprule
            $\epsilon$ & 1-shot   & 2-shot       & 4-shot       & 8-shot       & 12-shot \\
    \midrule
1            & \textbf{81.66} & 82.46 & 76.80 & 74.82 & 74.46 \\
4            & 79.90 & 82.62 & 79.84 & 76.04 & 68.52 \\
8           & 77.00 & \textbf{83.34} & 78.62 & 74.90 & 68.56 \\
$\infty$         & 80.86 & 80.48 & \textbf{80.94} & \textbf{81.44} & \textbf{75.48} \\
    \bottomrule
    \end{tabular}
    \caption{ICL accuracy under varying privacy budgets with different \textbf{$k$-shots} on DBPedia.}
    \label{tab:dbpedia_shots}
\end{table}

\begin{table}[h]
    \centering
    \small
    \begin{tabular}{llllll}
    \toprule
            $\epsilon$ & 1-shot   & 2-shot       & 4-shot       & 8-shot       & 12-shot \\
    \midrule
1          & 68.32 & 60.60 & 63.00 & 64.40 & 65.84 \\
4           & 68.76 & 59.72 & 62.48 & 63.08 & 63.28 \\
8            & \textbf{68.96} & \textbf{62.00} & \textbf{63.20} & 67.08 & 62.32 \\
$\infty$         & 62.40 & 54.04 & 61.92 & \textbf{67.16} & \textbf{70.12} \\
    \bottomrule
    \end{tabular}
    \caption{ICL accuracy under varying privacy budgets with different \textbf{$k$-shots} on TREC.}
    \label{tab:trec_shots}
\end{table}

We present a comprehensive analysis of in-context learning (ICL) performance across two datasets, namely DBPedia and TREC, under varying privacy budgets with \textit{different $k$-shot settings}. 

In Table \ref{tab:dbpedia_shots}, results on DBPedia show that performance generally improves with larger values of $\epsilon$, indicating that relaxing privacy constraints allows for more effective use of private data in generating synthetic demonstrations. However, as the number of shots increases, the performance can sometimes degrade slightly possible due to the lack of targeted private signal which can also be the training artifacts of LLM training. Finally, the best performance for 4-shot (80.94), 8-shot (81.44), and 12-shot (75.48) appears in the $\epsilon = \infty$ (non-private) case, highlighting the utility-privacy trade-off inherent to DP algorithms.

Table \ref{tab:trec_shots} presents the corresponding analysis on the TREC dataset, where the pattern across privacy levels is more nuanced. For example, $\epsilon = 8$ outperforms all other settings for it should be 1-shot, 2-shot, and 4-shot configurations, indicating that a moderately relaxed privacy guarantee yields significant utility benefits in certain tasks. 

In practical scenarios where preserving user privacy is critical, we find that $\epsilon = 1$ serves as a reasonable trade-off that offers strong privacy guarantees while still maintaining competitive performance. This makes it a suitable choice for many real-world applications requiring differentially private synthetic data generation. We also list synthetic examples in the Appendix.

\subsection{Structured Data Generation}
We evaluate how well our approach preserves privacy while generating syntactic structures from sensitive data. In structured generation tasks, many tokens are essential for maintaining the correct output structure. To assess this, we experiment on the WikiMoviesJSON generation task, using preprocessing and evaluation setups described in~\cite{amin2024private}. We assess performance using two metrics: (1) the percentage of outputs that are syntactically well-formed JSON parses, and (2) the percentage of outputs that pass basic schema validation. As shown in Table~\ref{tab:wiki-json}, our method achieves high-quality and schema-compliant JSON generation even under a strict privacy budget of $\epsilon = 1$, demonstrating the effectiveness of our approach for privacy-preserving structured text synthesis.

\begin{table}[!t]
    \centering
    \scalebox{0.8}{
    \begin{tabular}{cllrrr}
    \toprule
    \multicolumn{1}{c}{$\varepsilon$} & Method & $\tau$ & \multicolumn{1}{c}{Parses (\%)} & \multicolumn{1}{c}{Validates (\%)} & $\#$raw \\
    \midrule
    \multirow{3}{*}{1} & \multirow{2}{*}{\citet{amin2024private}} & 2 & $80.6_{1.3}$ & $74.2_{1.9}$ & $94.3_{1.2}$  \\
    &  & 2.5  & $4.9_{1.1}$ & $1.5_{0.1}$ & $138.0_{7.5}$ \\
    \cmidrule(r){2-6}
    & \textbf{Ours} & 1.13 & $84.2_{4.08}$ & $81.1_{8.1}$ & $10.3_{1.15}$ \\
    \bottomrule
    \end{tabular}
    }
    \caption{\small Results for generating JSON records from  \emph{WikiMoviesJSON} and report gains in structure preservation and validations. We report mean and standard deviation over 3 runs of dataset generation. Here, $\tau$ refers to the sampling temperature, and $\#$raw refers to the number of raw samples produced before parsing and validation checks.}
    \label{tab:wiki-json}
\end{table}

\subsection{Effect of Model Size and Architecture}
We evaluate in-context learning (ICL) performance across different LLM families and parameter scales. The results demonstrate a clear trend: larger models consistently achieve higher classification accuracy. This supports the hypothesis that both model size and architecture significantly influence ICL effectiveness.
\label{appendix:more-experiments:size}
% \begin{table}[h]
%     \centering
%     \begin{tabular}{cllr}
%     \toprule
%     \multicolumn{1}{c}{$\varepsilon$}  & Model & Acc. (\%)  \\
%     \midrule
%  & \texttt{google/gemma-3-1b-it} & $67.0$ \\
%  & \texttt{google/gemma-2-2b-it} & $84.9$ \\
%  & \texttt{meta-llama/Llama-3.2-3B} & $84.2$ \\
%  & \texttt{meta-llama/Meta-Llama-3-8B} & $89.9$ \\
%     \bottomrule
%     \end{tabular}
%     \caption{\small Results on DBPedia classification. Evaluation setup is the same as in \S\ref{sec:experiments}.}
%     \label{tab:size}
% \end{table}

\begin{table}[h]
    \centering
    \small
    \begin{tabular}{llr}
    \toprule
       Model & Acc. (\%)  \\
    \midrule
  \texttt{google/gemma3-1b-it} & $67.0$ \\
  \texttt{google/gemma2-2b-it} & $84.9$ \\
 \midrule
  \texttt{meta-llama/Llama-3.2-1B} & $53.9$ \\
  \texttt{meta-llama/Llama-3.2-3B} & $84.2$ \\
  \texttt{meta-llama/Meta-Llama3-8B} & $89.9$ \\
    \bottomrule
    \end{tabular}
    \caption{\small DBPedia classification accuracy for various LLMs. All evaluations use the same ICL setup described in \S\ref{sec:experiments} and the privacy budget is set to $\varepsilon = 1$.}
    \label{tab:size}
\end{table}

\subsection{Privacy Attacks}
While differential privacy offers theoretical privacy guarantees, empirical validation is essential~\citep{blanco2022critical}. We conduct personally identifiable information (PII) extraction attack on the Enron email dataset following the experimental setup described in~\cite{zeng2024good}. Specifically, we first construct a private dataset of text which contain email addresses in the body. Using Algorithm \ref{alg:main}, we construct a private subset $S$ and generate $T = 15$ synthetic tokens using the prompt template as \textit{``Extract only the email address from the above text.''}. The goal of this attack is to test whether our approach prevents release of private email addresses in the generated text. We evaluate under various privacy budgets, namely $\epsilon=[1, 4, 8]$. Across all tested privacy levels, we observe $zero$ sensitive email address in the generated responses demonstrating the effectiveness of our approach in preventing leakage of PII and demonstrate its strong practical privacy-preserving capabilities.

Furthermore, to assess the practical privacy of DP few-shot generation for in-context learning, we conduct membership inference attacks (MIA)~\cite{ShokriSSS17} following~\cite{duan2023flocks}. When actual private samples are used in prompts, attacks succeed with a high AUC (94.20 for $\epsilon=\infty$). In contrast, our DP approach significantly reduces AUC ($51.78$ for $\epsilon=4$), confirming improved membership privacy. See Appendix Table \textit{MIA} for details.

% \begin{table}[h!]
% \centering
% \begin{tabular}{|c|c|c|c|}
% \hline
% \(\epsilon\) & \(k = 10\) & \(k = 50\) & \(k = 100\) \\
% \hline
% 0.1 & 25 & 9 & 9 \\
% 0.5 & 29 & 11 & 6 \\
% 1.0 & 40 & 16 & 12 \\
% 2.0 & 35 & 15 & 22 \\
% 3.0 & 39 & 19 & 32 \\
% 4.0 & 61 & 39 & 35 \\
% 5.0 & 28 & 48 & 62 \\
% 10.0 & 13 & 15 & 16 \\
% \hline
% \end{tabular}
% \caption{Number of leaked emails for different \(\epsilon\) and \(k\)}
% \end{table}

% \begin{table}[h!]
% \centering
% \begin{tabular}{|c|c|}
% \hline
% \(\epsilon\) & Leaked emails \\
% \hline
% 5 & 1  \\
% 10 & 5  \\
% 15 & 60 \\
% 20 & 7  \\
% 25 & 5  \\
% 30 & 10 \\
% \(\infty\) & 7 \\
% \hline
% \end{tabular}
% \caption{Number of leaked emails for different privacy budget $\epsilon$.}
% \end{table}

% % \begin{table}[h!]
% % \centering
% % \begin{tabular}{|c|c|c|}
% % \hline
% % \(\epsilon\)  & \textbf{Leaked emails} & \textbf{Averaging} \\
% % \hline
% % 5 & 0 & \checkmark \\
% % 10 & 0 & \checkmark \\
% % 15 & 0 & \checkmark \\
% % 20 & 0 & \checkmark \\
% % 25 & 1 & \checkmark \\
% % 30 & 1 & \checkmark \\
% % $\infty$ & 1 & \checkmark \\
% % $\infty$ & 116 &  \texttimes \\

% % \hline
% % \end{tabular}
% % \caption{Eps, and Averaging Table}
% % \end{table}

\section{Discussion}
Traditional differentially private generation methods, such as those employed by \citet{amin2024private}, rely on metric-based mechanisms like the Sparse Vector Technique (SVT), which requires threshold computations under additive noise. These approaches typically employ distributional distance metrics such as  $\ell_1$ to score or select candidate tokens. However, these metrics are computed over normalized probability distributions (i.e., post-softmax), where even semantically similar tokens with slightly different probabilities or indices can yield high $\ell_1$ distance values. This issue is especially pronounced in high-dimensional output spaces of language models, where such metrics treat tokens as orthogonal dimensions, ignore semantic similarity, and are insensitive to token ranking that are crucial for meaningful text generation.

In contrast, our approach avoids the need for any explicit scoring and thresholding by adopting a simple averaging based mechanism. At each decoding step, we compute a mean of the clipped private aggregate and clipped public logits. This approach simplifies implementation and  eliminates the need to tune multiple sensitive hyperparameters (e.g., SVT threshold values, noise scales, temperature settings). Although this may attenuate some private signal, we find in practice that it provides a more stable, semantically meaningful, and privacy-preserving decoding procedure which is suited for downstream tasks where output coherence and content quality are paramount.

Moreover, approach by \citet{amin2024private} continues generating additional sentences beyond termination, often yielding low-quality synthetic examples. Our procedure strictly terminates when $\eos$ token is encountered. This guarantees that all outputs are semantically coherent and suitable for downstream tasks. Qualitative examples supporting this claim are provided in the Appendix.

\section{Conclusion}
As access to foundation models grows, the resources required for training these models have become expensive; hence, private prediction could emerge as a compelling alternative to private fine-tuning. In this work, we show that private prediction can generate synthetic text while following the standard differential-privacy guarantees. This privately generated corpus substantially boosts performance in many-shot in-context learning. Moreover, introducing a mechanism for sampling tokens from public models and blending them with private tokens enhances the utility of prediction tasks without compromising the privacy accounting.

\section{Limitations}
While our method is a practical implementation of private prediction to generate high-quality synthetic data, there will be a performance gap compared to private fine-tuning. Furthermore, fine-tuning-based approaches incur a privacy cost during training only, whereas private prediction methods pay a privacy penalty for every token generated during inference. Finally, any privacy-preserving method pays off via some loss of utility for improving privacy, and future research needs to close this gap. 

\section{Acknowledgment}
This material is based upon work supported by the Defense Advanced Research Projects Agency (DARPA) under Agreement No. HR0011-24-9-0424, the Advanced Research Projects Agency for Health (ARPA-H) under Contract Number SP4701-23-C-0073, and the National Science Foundation under Grant CCF-1900924. Any opinions, findings, and conclusions or recommendations expressed in this material are those of the author(s) and do not necessarily reflect the views of the Defense Advanced Research Projects Agency (DARPA), the Advanced Research Projects Agency for Health (ARPA-H), the National Science Foundation (NSF), or the United States Government.

\bibliography{main}
\newpage
% \input{ReproducibilityChecklist}
% \newpage
\appendix

\section{Hyperparameter tuning}
\label{sec:hyper}
This section describes our evaluation procedure and rationale for hyperparameter coupling decisions and excluded configurations. Based on initial experiments, we fix \( c = 10 \) and explore two temperature settings: low (\( \tau = 0.11 \)) and high (\( \tau = 2.1 \)). At high temperature, we observe text degeneration due to Gemma’s large vocabulary (256K) and clipping, which raises the probability floor of nonsensical tokens. Increasing the batch size \( \subsetsize \) reduces \( \varepsilon \) but increases the compute cost of decoding. Hence, we choose \( \subsetsize \) to meet a target \( \varepsilon \) while allowing generation of many examples at \( \tau = 1.02 \). For large \( \varepsilon \), setting \( \subsetsize \) too high becomes inefficient due to the resulting token volume and decoding cost.

\begin{table}[!h]
    \small
    \centering
    \begin{tabular}{lll}
    \toprule
    $\alpha$ & Description & Values \\
    \midrule
    \multirow{2}{*}{$\subsetsize$} & \multirow{2}{*}{batch size} & 15, 255, 380, \\
    & & 500 \\
    \midrule
    $c$ & logits clip bound & 10 \\
    $\tau$ & temperature & 0.131, 0.262, 1.048, 2.1 \\
    \bottomrule
    \end{tabular}
    \caption{\small Values for hyperparameters explored in this work.}
    \label{tab:hparams}
\end{table}

\section{Empirical privacy evaluation by membership inference attacks}
\label{sec:MIA}

\begin{table}[!h]
    \centering
    \begin{tabular}{ccccc}
    \toprule
        $\epsilon$ &4  &$\infty$  \\
        \midrule
          AUC & 51.78  &  94.20 \\
    \bottomrule
    \end{tabular}
    \caption{Empirical privacy evaluation  of our method for 1-shot ICL by MIA on Gemma 1.1 2B IT model.}
    \label{tab:MIA}
\end{table}

While differential privacy offers theoretical privacy guarantees, empirical evaluation is also crucial~\citep{blanco2022critical}. To assess the practical privacy of our DP few-shot generation for in-context learning, we conduct membership inference attacks~(MIA)~\citep{ShokriSSS17}, a practical method for measuring real-world privacy leakage.

We follow~\cite{duan2023flocks} to instantiate MIA in the in-context learning (ICL) setting, where the goal is to detect whether a data point was included in the LLM prompt. As expected, using true private samples leads to successful attacks. To evaluate our DP few-shot generation, we split the DBPedia dataset into member and non-member sets. Using member data, we generate 1-shot demonstrations with our DP algorithm for $\epsilon=4$ on 5 runs. For MIA, we query 50 member and 50 non-member samples, repeating for 100 trials to compute the average AUC. For the non-private baseline, we follow the same setup using actual member samples in the prompt.

Table~\ref{tab:MIA} shows the MIA results. Consistent with~\cite{duan2023flocks}, using actual private samples yields high AUC (94.20 for $\epsilon=\infty$), indicating successful attacks. In contrast, our DP approach significantly reduces the AUC, demonstrating improved membership privacy.

% \section{Effect of model size/architecture}
% We evaluate in-context learning (ICL) performance across different LLM families and parameter scales. The results demonstrate a clear trend: larger models consistently achieve higher classification accuracy. This supports the hypothesis that both model size and architecture significantly influence ICL effectiveness.
% \label{appendix:more-experiments:size}
% \input{tables/model_size}

\section{Generated Synthetic Examples}
\label{section:qualitative}

\begin{figure}[htbp]
\centering
\begin{minipage}{0.47\textwidth}
\begin{tcolorbox}
\smaller
The Belden Group's Belden Building is a prominent landmark building located in Chicago, Illinois. It is the current headquarters of The Belden Group, and features a unique and distinctive architecture that blends the architectural styles of various eras.
\end{tcolorbox}
\end{minipage}
\caption{A synthesis sample of Dbpedia for \textit{Building} category.}
\label{fig:icl-evaluation-dbpedia}
\end{figure}
\vspace{-0.5em}

\begin{figure}[htbp]
\centering
\begin{minipage}{0.47\textwidth}
\begin{tcolorbox}
\smaller
Emerging healthcare technologies are revolutionizing healthcare by making medical treatments more precise and efficient. Artificial intelligence-powered medical devices are assisting doctors in diagnosing diseases, while telehealth services are making healthcare more accessible to patients
\end{tcolorbox}
\end{minipage}
\caption{A synthesis sample of Agnews for \textit{Technology} category.}
\label{fig:icl-evaluation-agnews}
\end{figure}
\vspace{-0.5em}

\begin{figure}[!h]
\centering
\begin{minipage}{0.47\textwidth}
\begin{tcolorbox}
\smaller
What are the most common benefits of using a financial advisor?
\end{tcolorbox}
\end{minipage}
\caption{A synthesis sample of Trec for \textit{Description} question type.}
\label{fig:icl-evaluation-trec}
\end{figure}

\section{Public and Private Prompts}
We list all the  private and public prompts used for our experiments. Note that public generation prompts donot use any private information while querying the public LLMs.

\begin{figure}[htbp]
\begin{tcolorbox}
\begin{lstlisting}[language=Python]
#[User]
Generate only a text of news type {label}.

#[Assistant]
Text:
\end{lstlisting}
\end{tcolorbox}
\caption{\small Public Generation prompt for AGNEWS.}
\end{figure}

\begin{figure}[htbp]
\begin{tcolorbox}
\begin{lstlisting}[language=Python,]
Here are texts with News Type: {label}.

{text}

Please give me another one.

# [Assistant]
Text:
\end{lstlisting}
\end{tcolorbox}
\caption{\small Private Generation prompt for AGNEWS.}
\end{figure}

\begin{figure}[htbp]
\begin{tcolorbox}
\begin{lstlisting}[language=Python]
#[User]
Generate only a wiki entry of Category {label}.

#[Assistant]
Text:
\end{lstlisting}
\end{tcolorbox}
\caption{\small Public Generation prompt for DBPEDIA.}
\end{figure}

\begin{figure}[htbp]
\begin{tcolorbox}
\begin{lstlisting}[language=Python,]
# [User] 
Here are entries of Category: {label}.

{text}

Please give me another one.

# [Assistant]
Entry:
\end{lstlisting}
\end{tcolorbox}
\caption{\small Private Generation prompt for DBPEDIA.}
\end{figure}

\begin{figure}[htbp]
\begin{tcolorbox}
\begin{lstlisting}[language=Python]
#[User]
Generate only a question with Answer Type {label}.

#[Assistant]
Question:
\end{lstlisting}
\end{tcolorbox}
\caption{\small Public Generation prompt for TREC.}
\end{figure}

\begin{figure}[htbp]
\begin{tcolorbox}
\begin{lstlisting}[language=Python,]
# [User] 
Here are questions with Answer Type: {{label}}.

{{text}}

Please give me another one.

# [Assistant]
Question:
\end{lstlisting}
\end{tcolorbox}
\caption{\small Private Generation prompt for TREC.}
\end{figure}

\begin{figure}[htbp]
\begin{tcolorbox}
\begin{lstlisting}[language=Python]
#[User]
Give me text about a film and the extracted Phrase about its {field_name}. IMPORTANT: The exact {field_name} phrase "{keyword}" must be mentioned in Text.

# [Assistant]
Phrase: "{keyword}"
Text: "
\end{lstlisting}
\end{tcolorbox}
\caption{\small Public Generation prompt for MIT-G and MIT-D.}
\end{figure}

\begin{figure}[htbp]
\begin{tcolorbox}
\begin{lstlisting}[language=Python,]
# [User] 
Give me text about a film and the extracted Phrase about its {field_name}.
{text}

Please give me another Phrase and Text. IMPORTANT: The exact {field_name} phrase "{keyword}" must be mentioned in Text.

# [Assistant]
Phrase: "{keyword}"
Text: "
\end{lstlisting}
\end{tcolorbox}
\caption{\small Private Generation prompt for MIT-G and MIT-D.}
\end{figure}

\begin{figure}[htbp]
\begin{tcolorbox}
\begin{lstlisting}[language=Python,]
'Barabara,\nI had a lunch today with Rob Ladd from Duke (company, not university).\nHe is a Rice graduate and I mentioned to him the seminars that Enron was sponsoring.\nHe is willing to talk to you about substituting Duke for Enron as a sponsor of the \nseminar program. \nPlease, contact him at rtladd@duke-capitalpartners.com.\nHis cell phone number is 704 756 5354.\nI am working on the power price time series for you but I may run out of time.\nVince'
\end{lstlisting}
\end{tcolorbox}
\caption{\small A record from the Enron dataset.}
\end{figure}

\end{document}